\documentclass{article}

\usepackage{microtype}
\usepackage{graphicx}
\usepackage{booktabs} 
\usepackage{amssymb}
\usepackage{amsmath}
\usepackage{amsthm} 
\usepackage{float}
\usepackage{subcaption}
\usepackage{xcolor}
\usepackage{hyperref}
\usepackage{tablefootnote}

\newcommand{\Tau}{\mathcal{T}}

\DeclareMathOperator*{\argmin}{arg\,min}
\newtheorem{thm}{Theorem}

\usepackage[accepted]{icml2020}
\icmltitlerunning{Stateless Neural Meta-Learning using Second-Order Gradients}

\begin{document}

\twocolumn[
\icmltitle{Stateless Neural Meta-Learning using Second-Order Gradients}



\icmlsetsymbol{equal}{*}

\begin{icmlauthorlist}
\icmlauthor{Mike Huisman}{liacs}\footnotetext{test}
\icmlauthor{Aske Plaat}{liacs}
\icmlauthor{Jan N. van Rijn}{liacs}
\end{icmlauthorlist}

\icmlaffiliation{liacs}{Leiden Institute of Advanced Computer Science, Leiden University, Leiden, The Netherlands}

\icmlcorrespondingauthor{Mike Huisman}{m.huisman@liacs.leidenuniv.nl}

\icmlkeywords{Machine Learning, ICML}

\vskip 0.3in
]



\printAffiliationsAndNotice{}  

\begin{abstract}
Deep learning  typically requires large data sets and much compute power for each new problem that is learned.
Meta-learning can be used to learn a good prior that facilitates quick learning, thereby relaxing these requirements so that new tasks can be learned quicker; two popular approaches are MAML and the meta-learner LSTM.
In this work, we compare the two and formally show that the meta-learner LSTM subsumes MAML.
Combining this insight with recent empirical findings, we construct a new algorithm (dubbed TURTLE) which is simpler than the meta-learner LSTM yet more expressive than MAML.
TURTLE outperforms both techniques at few-shot sine wave regression and image classification on miniImageNet and CUB without any additional hyperparameter tuning, at a computational cost that is comparable with second-order MAML.
The key to TURTLE's success lies in the use of second-order gradients, which also significantly increases the performance of the meta-learner LSTM by $1$-$6$\% accuracy.
\end{abstract}

\section{Introduction}
\label{sec:intro}

Humans learn new tasks quickly. 
While deep neural networks have demonstrated human or even super-human performance on various tasks such as image recognition \cite{krizhevsky2012imagenet,he2015delving} and
game-playing \cite{mnih2015human, silver2016mastering},
learning a new task is generally slow and requires large amounts of data \cite{lecun2015deep}.
This limits their applicability in real-world domains where few data and limited computational resources are available. 

\textit{Meta-learning} \cite{schmidhuber1987evolutionary,schaul2010metalearning} is one approach to address this issue. 
The idea is to learn at two different levels of abstraction: at the \textit{outer-level} (across tasks), we learn a prior that facilitates faster learning at the \textit{inner-level} (single task) \cite{vilalta2002perspective, vanschoren2018meta, hospedales2020meta, huisman2021}. 
The prior that we learn at the outer-level can take on many different forms, such as the learning rule \cite{andrychowicz2016learning, ravi2017optimization} and the weight initialization \cite{nichol2018reptile, finn2017model}. 

MAML \cite{finn2017model} and the meta-learner LSTM \cite{ravi2017optimization} are two well-known techniques that focus on these two types of priors. 
More specifically, MAML aims to learn a good weight initialization from which it can learn new tasks quickly using regular gradient descent. 
In addition to learning a good weight initialization, the meta-learner LSTM \cite{ravi2017optimization}  attempts to learn the optimization procedure in the form of a separate LSTM network. 
The meta-learner LSTM is more general than MAML in the sense that the LSTM can learn to perform gradient descent (see \autoref{sec:turtle}) or something better.

This suggests that the performance of MAML can be mimicked by the meta-learner LSTM on few-shot image classification.
However, our experimental results and those by \citet{finn2017model} show that this is not necessarily the case.
The meta-learner LSTM fails to find a solution in the \textit{meta-landscape} that learns  as well as gradient descent. 

To improve upon the meta-learner LSTM, we introduce TURTLE, which uses a fully-connected feed-forward network as an optimizer---a meta-network that is simpler than an LSTM. 
And, because it uses a meta-network, TURTLE is more expressive than MAML, as the meta-network can learn to perform gradient descent.
We empirically demonstrate that TURTLE outperforms both of these techniques at few-shot sine wave regression and, without additional hyperparameter tuning, exceeds their performance in various settings involving the commonly used miniImageNet \citep{vinyals2016matching} and CUB \cite{wah2011caltech} benchmarks.
Our contributions are:
\begin{itemize}
    \item We formally show that the meta-learner LSTM subsumes MAML.
    \item We formulate a new meta-learning algorithm called TURTLE which uses a simpler meta-network than the meta-learner LSTM and is more expressive than MAML. 
    \item We demonstrate that TURTLE outperforms both MAML and the meta-learner LSTM on sine wave regression, and various settings involving miniImageNet and CUB by at least $1$\% accuracy without any additional hyperparameter tuning. TURTLE requires roughly the same amount of computation time as second-order MAML. 
    \item Based on the results of TURTLE, we enhance the meta-learner LSTM by using raw gradients as meta-learner input and second-order information and show these changes result in a performance boost of $1$-$6$\% accuracy.
\end{itemize}

\section{Related work}
\label{sec:background}

The success of deep learning techniques has been largely limited to domains where abundant data and large compute resources are available \cite{lecun2015deep}. 
The reason for this is that learning a new task requires large amounts of resources.
Meta-learning is an approach that holds the promise of relaxing these requirements by learning to learn.
The field has attracted much attention in recent years.

One popular technique is MAML \cite{finn2017model} which aims to find a good weight initialization from which new tasks can be learned quickly within several gradient update steps. 
Many works build upon the key idea of MAML, for example, to decrease the computational costs \cite{nichol2018reptile, rajeswaran2019meta}, increase the applicability to online and active learning settings \cite{grant2018recasting, finn2018probabilistic}, or increase the expressivity of the algorithm \cite{li2017metasgd, park2019meta, lee2018gradient}. 
Despite its popularity, MAML does no longer yield state-of-the-art performance on few-shot learning benchmarks \cite{lu2020learning}, as it is surpassed by, for example, latent embedding optimization (LEO) \cite{rusu2018meta} which optimizes the initial weights in a lower-dimensional latent space, and MetaOptNet \cite{lee2019meta}, which stacks a convex model on top of the meta-learned initialization of a high-dimensional feature extractor.
However, although these approaches achieve state-of-the-art techniques on few-shot benchmarks, MAML is more elegant and more generally applicable as it can also be used in reinforcement learning settings \cite{finn2017model}. 

While  the meta-learner LSTM \cite{ravi2017optimization} learns both an initialization and an optimization procedure, it is generally hard to properly train the optimizer \citep{metz2019understanding}.
As a result, techniques that use hand-crafted learning rules instead of trainable optimizers may yield better performance.    
It is perhaps for this reason that most meta-learning algorithms use simple, hand-crafted optimization procedures to learn new tasks, such as regular gradient descent \cite{bottou2004}, Adam \cite{kingma2015adam}, or RMSprop \cite{tieleman2017divide}.
\citet{andrychowicz2016learning}, however, show that learned optimizers may learn faster and yield better performance than gradient descent. 

The goal of our work is to show that, despite the practical difficulties of the meta-learner LSTM, MAML \emph{can} be outperformed by learning a separate meta-network.\footnote{Our code can be found at: \url{https://github.com/mikehuisman/revisiting-learned-optimizers}} 
Our technique, dubbed TURTLE, replaces the LSTM module from the meta-learner LSTM with a feed-forward neural network.

Note that \citet{metz2019understanding} also used a regular feed-forward network as an optimizer.
However, they were mainly concerned with understanding and correcting the difficulties that arise from training an optimizer and do not learn a weight initialization for the base-learner network as we do. 
\citet{baik2020meta} also use a feed-forward network on top of MAML but its goal is to generate a per-step learning rate and weight decay coefficients.
The feed-forward network in TURTLE, in contrast, generates direct weight updates.

\section{Preliminaries}\label{sec:prelims}

In this section, we explain the notation and the concepts of the works that we build upon.

\subsection{Few-shot learning}\label{sec:fewshotlearning}
In the context of supervised learning, the few-shot setup is commonly used as a testbed for meta-learning algorithms \cite{vinyals2016matching, finn2017model, nichol2018reptile, ravi2017optimization}. 
One reason for this is the fact that tasks $\Tau_j$ are small, which makes learning a prior \textit{across} tasks not overly expensive. 

Every task $\Tau_j$ consists of a support (training) set $D^{tr}_{\Tau_j}$ and query (test) set $D^{te}_{\Tau_j}$ \cite{vinyals2016matching, lu2020learning, ravi2017optimization}. 
When a model is presented with a new task, it tries to learn the associated concepts from the support set. 
The success of this learning process is then evaluated on the query set.
Naturally, this means that the query set contains concepts that were present in the support set. 

In classification settings, a commonly used instantiation of the few-shot setup is called $N$-way $k$-shot learning \cite{finn2017model, vinyals2016matching}.
Here, given a task $\Tau_j$, every support set contains $k$ examples for each of the $N$ distinct classes. 
Moreover, the query set must contain examples from one of these $N$ classes.

Suppose we have a dataset $D$ from which we can extract $J$ tasks. 
For meta-learning purposes, we split these tasks into three non-overlapping partitions: (i)~meta-training, (ii)~meta-validation, and (iii)~meta-test tasks \cite{ravi2017optimization, sun2019meta}. 
These partitions are used for training the meta-learning algorithm, hyperparameter tuning, and evaluation, respectively.
Note that non-overlapping means that every partition is assigned some class labels which are unique to that partition.  

\subsection{MAML}

As mentioned before, MAML \cite{finn2017model} attempts to learn a set of initial neural network parameters $\boldsymbol{\theta}$ from which we can quickly learn new tasks within $T$ steps of gradient descent, for a small value of $T$.
Thus, given a task $\Tau_j = (D^{tr}_{\Tau_j}, D^{te}_{\Tau_j})$, MAML will produce a sequence of weights $(\boldsymbol{\theta}_j^{(0)}, \boldsymbol{\theta}^{(1)}_j, \boldsymbol{\theta}^{(2)}_j, ..., \boldsymbol{\theta}^{(T)}_j)$, where 
\begin{align}
    \boldsymbol{\theta}^{(t+1)}_j = \boldsymbol{\theta}_j^{(t)} - \alpha \nabla_{\boldsymbol{\theta}^{(t)}_j} \mathcal{L}_{D^{tr}_{\Tau_j}}(\boldsymbol{\theta}^{(t)}_j). \label{eq:graddescent}
\end{align}
Here, $\alpha$ is the inner learning rate and $\mathcal{L}_{D}(\boldsymbol{\varphi})$ the loss of the network with weights $\boldsymbol{\varphi}$ on dataset $D$.   
Note that the first set of weights in the sequence is equal to the initialization, i.e., $\boldsymbol{\theta}_j^{(0)} = \boldsymbol{\theta}$.

Given a distribution of tasks $p(\Tau)$, we can formalize the objective of MAML as finding the initial parameters
\begin{align}
    \argmin_{\boldsymbol{\theta}} \mathbb{E}_{\Tau_j \backsim p(\Tau)} \left[  \mathcal{L}_{D^{te}_{\Tau_j}}(\boldsymbol{\theta}^{(T)}_j) \right]. \label{eq:mamlobj}   
\end{align}
Note that the loss is taken with respect to the query set, whereas $\boldsymbol{\theta}^{(T)}_j$ is computed on the support set $D^{tr}_{\Tau_J}$.

The initialization parameters $\boldsymbol{\theta}$ are updated by optimizing this objective in \autoref{eq:mamlobj}, where the expectation over tasks is approximated by sampling a batch of tasks.
Importantly, updating these initial parameters requires  backpropagation through the optimization trajectories on the tasks from the batch.
This implies the computation of second-order derivatives, which is computationally expensive. 
However, \citet{finn2017model} have shown that first-order MAML, which ignores these higher-order derivatives and is computationally less demanding, works just as well as the complete, second-order MAML version.

\subsection{Meta-learner LSTM}
\label{sec:metalstm}

The meta-learner LSTM by \citet{ravi2017optimization} can be seen as an extension of MAML as it does not only learn the initial parameters $\boldsymbol{\theta}$ but also the optimization procedure which is used to learn a given task. 
Note that MAML only uses a single base-learner network, while the meta-learner LSTM uses a separate meta-network to update the base-learner parameters. 
Thus, instead of computing $(\boldsymbol{\theta}_j^{(0)}, \boldsymbol{\theta}^{(1)}_j, \boldsymbol{\theta}^{(2)}_j, ..., \boldsymbol{\theta}^{(T)}_j)$ using regular gradient descent as done by MAML, the meta-learner LSTM learns a procedure that can produce such a sequence of updates, using a separate meta-network. 

This trainable optimizer takes the form of a special LSTM module, which is applied to every weight in the base-learner network after the gradients and loss are computed on the support set. 
The idea is to embed the base-learner weights into the cell state $\boldsymbol{c}$ of the LSTM module. 
Thus, for a given task $\mathcal{T}_j$, we start with cell state $\boldsymbol{c}_j^{(0)} = \boldsymbol{\theta}$. 
After this initialization phase, the base-learner parameters (which are now inside the cell state) are updated as
\begin{align}
    &\boldsymbol{c}_j^{(t+1)} = \nonumber \\ 
    &\sigma \left( \boldsymbol{W}_{\boldsymbol{f}} \cdot [\nabla_{\boldsymbol{\theta}^{(t)}_j}, \mathcal{L}_{D^{tr}_{\mathcal{T}_j}}, \boldsymbol{\theta}^{(t)}_j, \boldsymbol{f}^{(t-1)}_j ] + \boldsymbol{b}_{\boldsymbol{f}}  \right)  \odot \boldsymbol{c}_j^{(t)} \nonumber\\ 
    & +\sigma \left( \boldsymbol{W}_{\boldsymbol{i}} \cdot [\nabla_{\boldsymbol{\theta}^{(t)}_j}, \mathcal{L}_{D^{tr}_{\mathcal{T}_j}}, \boldsymbol{\theta}^{(t)}_j, \boldsymbol{i}^{(t-1)}_j ] + \boldsymbol{b}_{\boldsymbol{i}}  \right) \odot \bar{\boldsymbol{c}}^{(t)}_j,  \label{eq:lstmupdate}
\end{align}
where $\odot$ is the element-wise product, the two sigmoid factors $\sigma$ are the parameterized forget gate $\boldsymbol{f}^{(t)}_j$ and learning rate $\boldsymbol{i}^{(t)}_j$ vectors that steer the learning process, $\nabla_{\boldsymbol{\theta}^{(t)}_j} = \nabla_{\boldsymbol{\theta}^{(t)}_j}\mathcal{L}_{D^{tr}_{\mathcal{T}_j}}(\boldsymbol{\theta}^{(t)}_j)$, $\mathcal{L}_{D^{tr}_{\mathcal{T}_j}}=\mathcal{L}_{D^{tr}_{\mathcal{T}_j}}(\boldsymbol{\theta}^{(t)}_j)$, and $ \bar{\boldsymbol{c}}^{(t)}_j = -\nabla_{\boldsymbol{\theta}_j^{(t)}} \mathcal{L}_{D^{tr}_{\mathcal{T}_j}}(\boldsymbol{\theta}^{(t)}_j)$.
Both the learning rate and forget gate vectors are parameterized by weight matrices $\boldsymbol{W}_f, \boldsymbol{W}_i$ and bias vectors $\boldsymbol{b}_{\boldsymbol{f}}$ and $\boldsymbol{b}_{\boldsymbol{i}}$, respectively.
These parameters steer the inner learning on tasks and are updated using regular, hand-crafted optimizers after every meta-training task. 
As noted by \citet{ravi2017optimization}, this is equivalent to gradient descent when $\boldsymbol{c}^{(t)} = \boldsymbol{\theta}^{(t)}_j$, and the sigmoidal factors are equal to $\boldsymbol{1}$ and $\boldsymbol{\alpha}$, respectively. 

In spite of the fact that the LSTM module is applied to every weight individually to produce updates, it does maintain a separate hidden state for each of them. 
In a similar fashion to MAML, updating the initialization parameters (and LSTM parameters) would require propagating backwards through the optimization trajectory for each task.
To circumvent the computational costs associated with this expensive operation, the meta-learner LSTM assumes that input gradients and losses are \textit{independent} of the parameters in the LSTM.

\section{Towards stateless neural meta-learning}\label{sec:turtle}
In this section, we study the theoretical relationship between MAML and the meta-learner LSTM.
Based on the resulting insight, we formulate a new meta-learning algorithm called TURTLE (s\underline{t}ateless ne\underline{ur}al me\underline{t}a-\underline{le}arning) which is simpler than the meta-learner LSTM and more expressive than MAML.

\subsection{Theoretical relationship}
There is an obvious relationship between MAML and the meta-learner LSTM.
Instead of being restricted to using gradient descent, the meta-learner LSTM uses a meta-network that can learn the optimization procedure to learn new tasks.
This observation leads to  the following theorem: 

\begin{thm}
The meta-learner LSTM subsumes MAML
\end{thm}

\begin{proof}
We prove this theorem through mathematical induction on the updates made by MAML and the meta-learner LSTM. 

Suppose that both MAML and the meta-learner LSTM attempt to learn an arbitrary task $\mathcal{T}_j$.
Without loss of generality, let us furthermore assume that both techniques are allowed to make $T \in \mathbb{N}$ updates on the support set $D^{tr}_{\mathcal{T}_j}$ and that MAML does this using gradient descent with a learning rate denoted by $\alpha$.  

To satisfy the base case, we can simply assume that MAML and the meta-learner LSTM start with the same initial parameters $\boldsymbol{\theta}$, i.e., $\boldsymbol{c}^{(0)}_j = \boldsymbol{\theta}^{(0)}_j = \boldsymbol{\theta}$. 

To show that the meta-learner LSTM could learn gradient descent as inner learning procedure, and thereby complete the proof, we have to show for $t = 0,...,T-1$ that $\boldsymbol{c}^{(t+1)}_j = \boldsymbol{\theta}^{(t+1)}_j$.
Substituting both the right- and left-hand sides using \autoref{eq:graddescent} and \autoref{eq:lstmupdate}, this means that we have to show the possibility that
\begin{align}
    \boldsymbol{f}^{(t)}_j \odot \boldsymbol{c}_j^{(t)} + \boldsymbol{i}^{(t)}_j \odot \bar{\boldsymbol{c}}^{(t)}_j &= \boldsymbol{\theta}_j^{(t)} - \alpha \nabla_{\boldsymbol{\theta}^{(t)}_j} \mathcal{L}_{D^{tr}_{\Tau_j}}(\boldsymbol{\theta}^{(t)}_j),
\end{align}
where we used the shorthands 
\begin{align}
    \boldsymbol{f}_j^{(t)} = \sigma \left( \boldsymbol{W}_{\boldsymbol{f}} \cdot [\nabla_{\boldsymbol{\theta}^{(t)}_j}, \mathcal{L}_{D^{tr}_{\mathcal{T}_j}}, \boldsymbol{\theta}^{(t)}_j, \boldsymbol{f}^{(t-1)} ] + \boldsymbol{b}_{\boldsymbol{f}}  \right), \label{eq:deff}\\
    \boldsymbol{i}_j^{(t)} = \sigma \left( \boldsymbol{W}_{\boldsymbol{i}} \cdot [\nabla_{\boldsymbol{\theta}^{(t)}_j}, \mathcal{L}_{D^{tr}_{\mathcal{T}_j}}, \boldsymbol{\theta}^{(t)}_j, \boldsymbol{i}^{(t-1)} ] + \boldsymbol{b}_{\boldsymbol{i}}  \right).
\end{align}

By design, we have that $\bar{\boldsymbol{c}}^{(t)}_j = -\nabla_{\boldsymbol{\theta}^{(t)}_j} \mathcal{L}_{D^{tr}_{\Tau_j}}(\boldsymbol{\theta}^{(t)}_j)$. 
Moreover, by our inductive hypothesis, we also know that $\bar{\boldsymbol{c}}^{(t)}_j = \boldsymbol{\theta}_j^{(t)}$.
Thus, it remains to be shown that it is possible to have $\boldsymbol{f}^{(t)}_j = \boldsymbol{1}$ and $\boldsymbol{i}^{(t)}_j = \boldsymbol{a}$. 

Starting with the former, we have that
\begin{align}
    \boldsymbol{1} &= \boldsymbol{f}^{(t)}_j \\
    &=  \sigma \left( \boldsymbol{W}_{\boldsymbol{f}} \cdot [\nabla_{\boldsymbol{\theta}^{(t)}_j}, \mathcal{L}_{D^{tr}_{\mathcal{T}_j}}, \boldsymbol{\theta}^{(t)}_j, \boldsymbol{f}^{(t-1)} ] + \boldsymbol{b}_{\boldsymbol{f}}  \right) \label{eq:step1}\\
    &= \sigma \left( \boldsymbol{b}_{\boldsymbol{f}}  \right) \label{eq:step2} \\
    &=\left[ \frac{1}{1+e^{-b_{\boldsymbol{f}}^{(1)}}} \,\,...\,\, \frac{1}{1+e^{-b_{\boldsymbol{f}}^{(n)}}}   \right].\label{eq:step3}
\end{align}
Here, \autoref{eq:step1} is due to \autoref{eq:deff} and \autoref{eq:step2} by assuming that $\boldsymbol{W}_{\boldsymbol{f}}$ is a matrix of zeros, which we can do as our only task is to show the mere possibility of a parameterization of $\boldsymbol{W}_{\boldsymbol{f}}$ and $\boldsymbol{b}_{\boldsymbol{f}}$ that yields similar behavior as gradient descent.
\autoref{eq:step3} follows from the definition of the sigmoidal function.
In the last expression, $n$ denotes the number of base-learner parameters.
It is easy to see that this equation holds for sufficiently large $b_{\boldsymbol{f}}^{(m)}$ for $m \in \{ 1,...,n \}$.

Lastly, we show the possibility that $\boldsymbol{i}^{(t)}_j = \boldsymbol{a}$.
Leveraging the symmetry between the definitions of $\boldsymbol{i}^{(t)}_j$ and $\boldsymbol{f}^{(t)}_j$, we use the above derivation and conclude that we have to show that it is possible to have $\boldsymbol{\alpha} = \sigma (\boldsymbol{b}_{\boldsymbol{i}})$.
It is straightforward to show that this equation is satisfied when
\begin{align}
    b^{(m)}_{\boldsymbol{i}} = -\ln \left( \frac{1 -\alpha }{\alpha}  \right)
\end{align}
for $m \in \{ 1,...,n \}$.
Since $0 < \alpha < 1$, this is indeed possible and thus the proof is complete.
\end{proof}

\subsection{Potential problems of the meta-learner LSTM}
The theoretical insight that meta-learner LSTM subsumes MAML is not congruent with empirical findings which show that MAML outperforms the meta-learner LSTM on the miniImageNet image classification benchmark \cite{finn2017model, ravi2017optimization}, indicating that LSTM is unable to successfully navigate the error landscape to find a solution at least as good as the one found by MAML. 

A potential cause  is that the meta-learner LSTM attempts to learn a \textit{stateful} optimization procedure.
That is, the LSTM module has to learn state dynamics that allow for quick learning of new tasks. 
Learning such a stateful procedure requires additional parameters which may increase the complexity of the meta-landscape.  
We conjecture that removing the stateful nature of the trainable optimizer may smoothen the meta-landscape and allow for finding better solutions. 
For this reason, we replace the LSTM module with a regular fully-connected, feed-forward network.

Another potential cause of the underperformance could be the first-order assumption made by the meta-learner LSTM, which we briefly mentioned in \autoref{sec:metalstm}. 
Effectively, this disconnects the computational graph by stating that weight updates made at time step $t$ by the meta-network do not influence the inputs that this network receives at future time steps $t < t' < T$.
It has to be noted that this assumption and its consequences are of a substantially different nature than the assumption made by first-order MAML, which was shown to leave the performance mostly unaffected \citep{finn2017model}. 
More specifically, first-order MAML ignores the optimization trajectory and simply updates the base-learner parameters $\boldsymbol{\theta}$ in the opposite direction of the gradients of the query set loss, i.e., $-\nabla_{\boldsymbol{\theta}^{(T)}_j} \mathcal{L}_{D^{te}_{\mathcal{T}_j}}(\boldsymbol{\theta}^{(T)}_j)$.
When the base-level landscape is reasonably locally smooth, these query set update directions point towards the locally optimal parameters of the task.
In turn, first-order MAML moves the initialization parameters closer to these locally optimal parameters. 
It is thus intuitive that this first-order assumption is mostly harmless.
For the meta-learner LSTM, however, we fail to identify such an intuition that justifies the first-order assumption. 
Moreover, \citet{ravi2017optimization} have not shown an  empirical justification for this assumption. 

\subsection{TURTLE}\label{sec:metanetwork}

In an attempt to make the meta-landscape easier to navigate, we introduce a new algorithm, TURTLE, which trains a feed-forward meta-network to update the base-learner parameters. 
TURTLE is simpler than the meta-learner LSTM as it uses a \textit{stateless} feed-forward neural network as a trainable optimizer, yet more expressive than MAML as its meta-network can learn to perform gradient descent.

The trainable optimizer in TURTLE is thus a fully-connected feed-forward neural network. 
We denote the batch of inputs that this network receives at time step $t$ in the inner loop for task $\mathcal{T}_j$ as $I_j^{(t)} \in \mathbb{R}^{n \times d}$, where $n$ and $d$ are the number of base-learner parameters and the dimensionality of the inputs, respectively. 
The exact inputs that this network receives will be determined empirically, but two choices, inspired by the meta-learner LSTM, are: (i)~the gradients with respect to all parameters and (ii)~the current loss (repeated $n$ times for each parameter in the base-network).  

Moreover, we could mitigate the absence of a state in the meta-network by including a time step $t \in \{ 0,1,...,T-1 \}$ and/or historical information such as a moving average of previous gradients or updates made by the meta-network. 
We denote the latter by $\boldsymbol{h}_j^{(t)}$ which is updated by
\begin{align}
    \boldsymbol{h}_j^{(t+1)} = \beta \boldsymbol{h}_j^{(t)} + (1 - \beta) \boldsymbol{v}_j^{(t)},
\end{align}
where $0 \leq \beta \leq 0$ is a constant that determines the time span over which previous inputs affect the new state $\boldsymbol{h}_j^{(t+1)}$, and $\boldsymbol{v}_j^{(t)} \in \mathbb{R}^n$ is the new information (either the updates or gradients at time step $t$).
When using previous updates, we  initialize $\boldsymbol{h}_j^{(0)}$ by a vector of zeros. 

Weight updates are then computed as follows
\begin{align}
    \boldsymbol{\theta}^{(t+1)}_j = \boldsymbol{\theta}_j^{(t)} + \boldsymbol{\alpha} \odot  g_{\boldsymbol{\phi}}(I_j^{(t)}),   \label{eq:turtledate}
\end{align}
where $\boldsymbol{\alpha} \in \mathbb{R}^n$ is a vector of learning rates per parameter. 
Note that this weight update equation is simpler than the one used by the meta-learner LSTM (see \autoref{eq:lstmupdate}) as our meta-network $g_{\boldsymbol{\phi}}$ is stateless.
Therefore, we do not have parameterized forget and input gates. 
Moreover, the learning rates per parameter in $\boldsymbol{\alpha}$ are not constrained to be within the interval $[0,1]$ as is the case for the meta-learner LSTM due to the use of the sigmoid function.

In \autoref{alg:metalearn} we show, in different colors, the code for MAML (red), the meta-learner LSTM (blue), and TURTLE (green). Although the code structure of the three meta-learners is  similar,  the update rules are quite different.  
Both the base- and meta-learner parameters $\boldsymbol{\theta}$ and $\boldsymbol{\phi}$ are updated by backpropagation through the optimization trajectories (line 11).

\begin{algorithm}[htb]
   \caption{\colorbox{red!30}{MAML}
   \colorbox{blue!30}{meta-learner LSTM}  \colorbox{green!30}{TURTLE}}
   \label{alg:gml}
\begin{algorithmic}[1]
   \STATE Initialize parameters $\boldsymbol{\Theta} =$ 
   \colorbox{red!30}{$\{\boldsymbol{\theta}\}$}
   \colorbox{blue!30}{$\{\boldsymbol{\theta}, \boldsymbol{\phi} \}$}  \colorbox{green!30}{$\{\boldsymbol{\theta}, \boldsymbol{\phi} \}$}
   \STATE Initialize $g_{\boldsymbol{\phi}}$ as 
   \colorbox{red!30}{N.A.}
   \colorbox{blue!30}{LSTM}  \colorbox{green!30}{feed-forward network}
   \REPEAT
        \STATE Sample batch of J tasks $B = \{ \Tau_j \backsim p(\Tau) \}_{j=1}^J$
        \FOR{$\Tau_j = (D^{tr}_{\Tau_j}, D^{te}_{\Tau_j})$ in $B$}
            \STATE $\boldsymbol{\theta}^{(0)}_j = \boldsymbol{\theta}$
            \FOR{$t = 1,...,T$}
                \STATE Update $\boldsymbol{\theta}^{(t)}_j$ using 
                \colorbox{red!30}{Eq.~\ref{eq:graddescent}}
   \colorbox{blue!30}{Eq.~\ref{eq:lstmupdate}}  \colorbox{green!30}{Eq.~\ref{eq:turtledate}}
            \ENDFOR
        \ENDFOR
        \STATE Update $\boldsymbol{\Theta}$ using $\sum_{\Tau_j \in B} \mathcal{L}_{D^{te}_{\Tau_j}}(\boldsymbol{\theta}^{(T)}_j)$
   \UNTIL{convergence}
\end{algorithmic}
\label{alg:metalearn}
\end{algorithm}

\section{Experiments}\label{sec:experiments}
In this section, we describe our experimental setup and the results that we obtained. 

\subsection{Problem descriptions}
For our experimental evaluation, we use two problems: sine wave regression and image classification.

\begin{figure*}
\vskip 0.2in
\begin{center}
\begin{subfigure}[b]{0.32\linewidth}
\centering
\includegraphics[width=\linewidth]{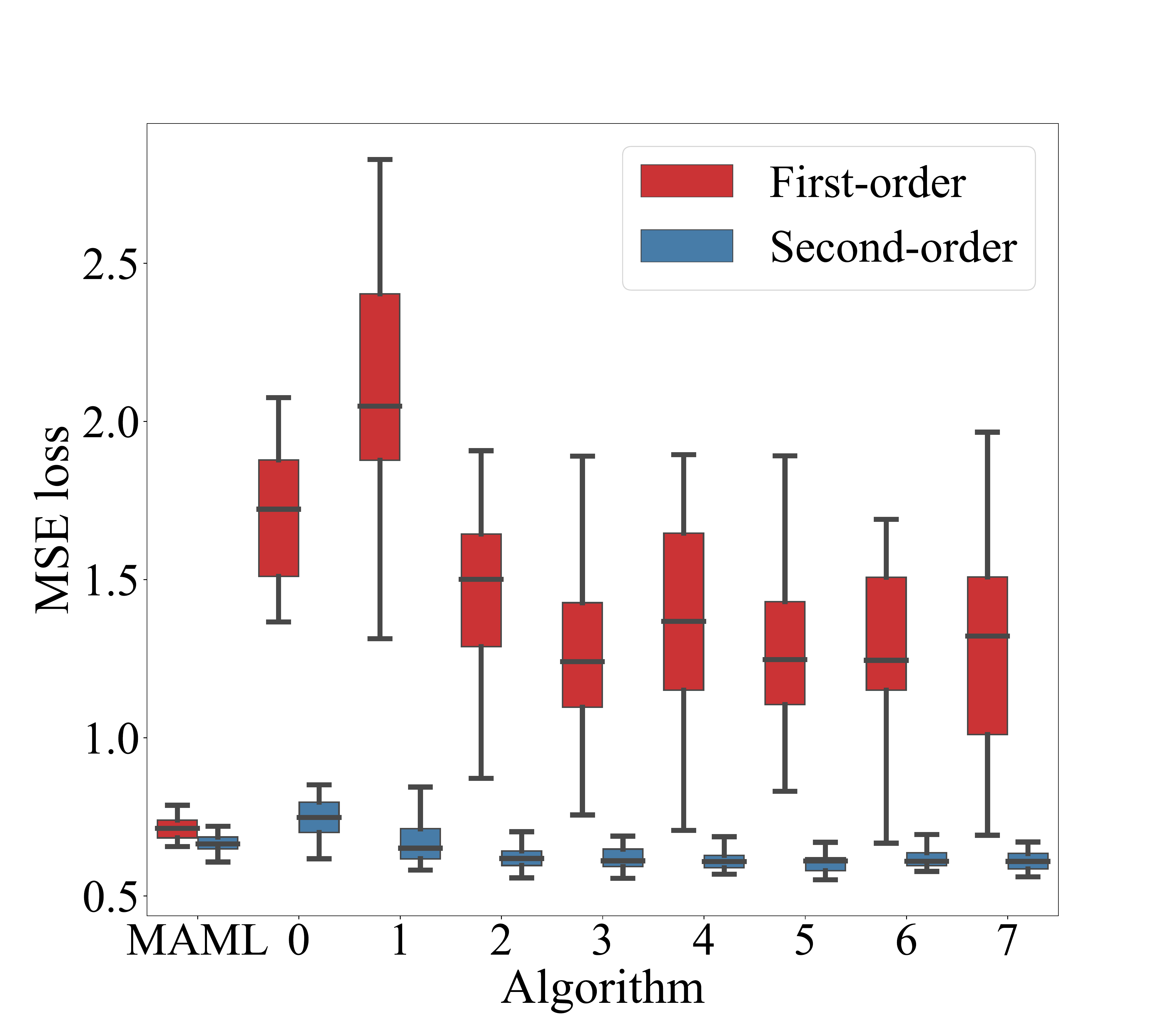}
\caption{$T = 1$}
\end{subfigure}
\begin{subfigure}[b]{0.32\linewidth}
\centering
\includegraphics[width=\linewidth]{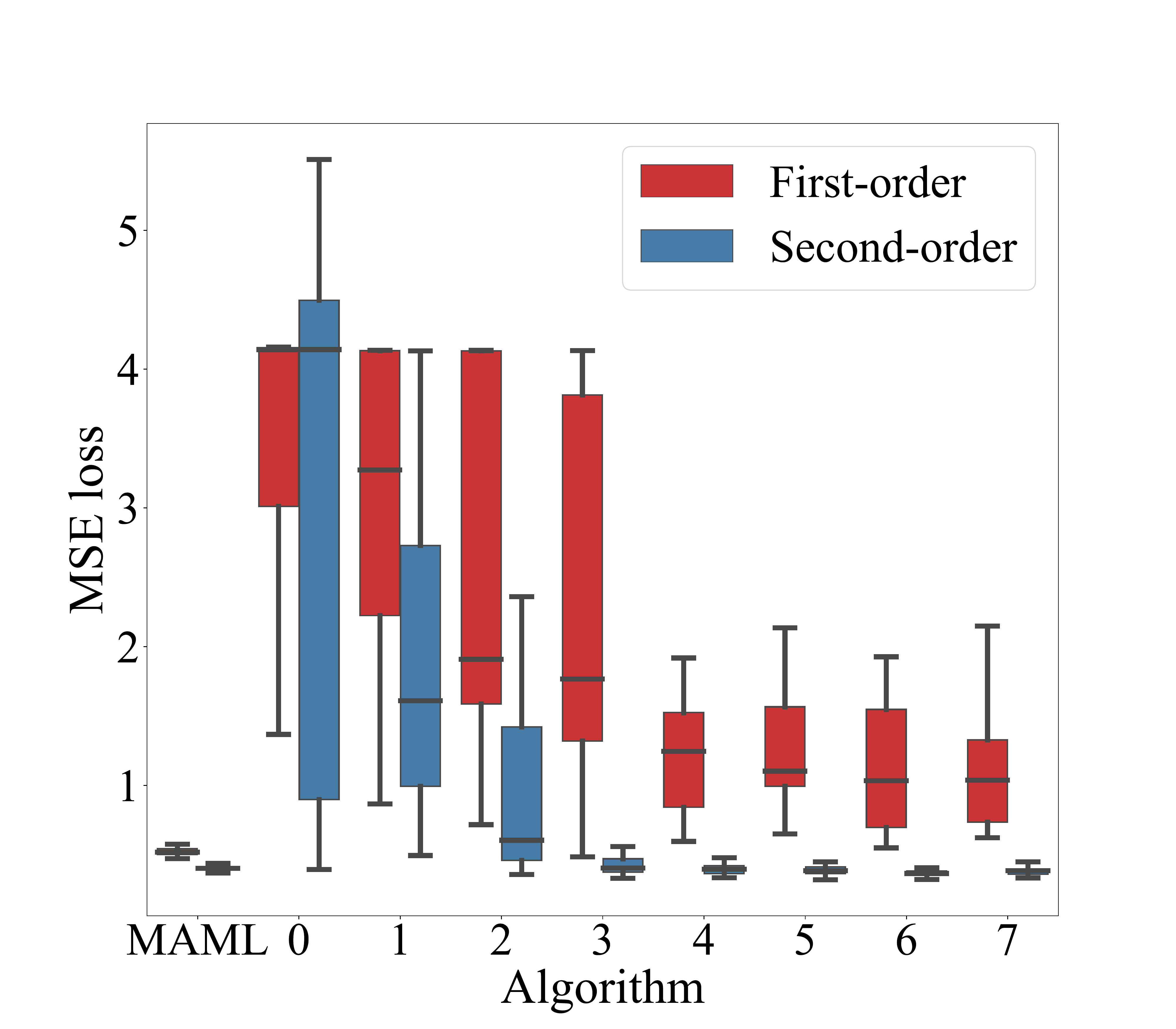}
\caption{$T = 5$}
\end{subfigure}
\begin{subfigure}[b]{0.32\linewidth}
\centering
\includegraphics[width=\linewidth]{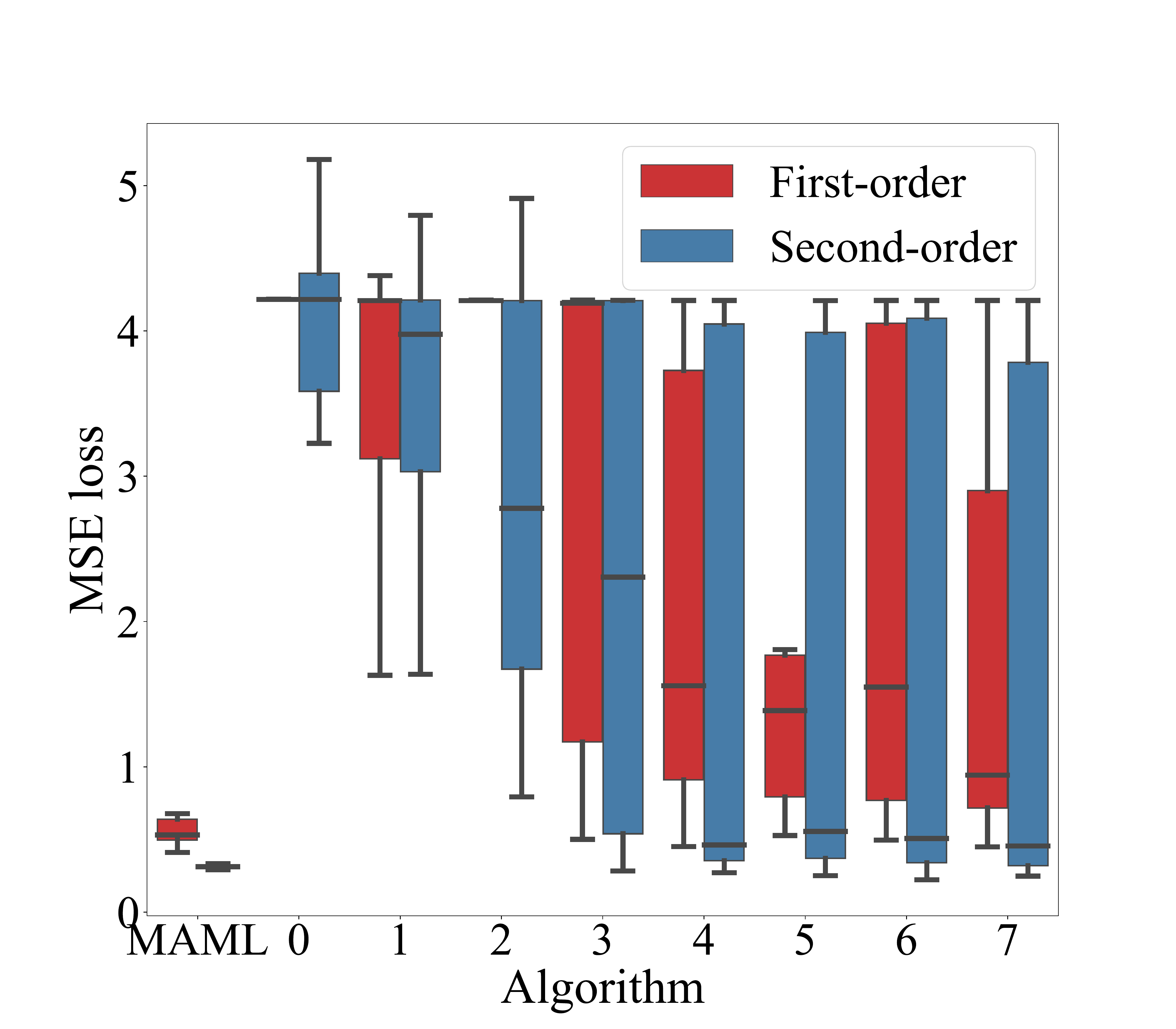}
\caption{$T = 10$}
\end{subfigure}
\caption{Influence of the order, number of update steps, and number of hidden layers (horizontal axis) on the meta-validation performance of TURTLE on $5$-shot sine wave regression. We also plot the performance of first- and second-order MAML for comparison. Note that a lower MSE loss corresponds to better performance.}
\label{fig:layersorder}
\end{center}
\vskip -0.2in
\end{figure*}

\subsubsection{Sine wave regression} 

The sine wave regression  problem was originally proposed by \citet{finn2017model}. 
In this setup, every task $\mathcal{T}_j$ corresponds to a different sine wave function $s_{j}(x) = a \cdot \sin (x - p)$, where $a \in [0.1, 5.0]$ is the amplitude, and $p \in [0, \pi]$ the phase.
Both of these parameters are selected uniformly at random from their corresponding ranges.

Given a task $\mathcal{T}_{j} = (D^{tr}_{\mathcal{T}_j}, D^{te}_{\mathcal{T}_j})$, the goal for the base-learner network is to infer the sine wave that gives rise to observations $(x_{i},y_{i})$ in the support set of a task.
The quality of this inference is determined with the mean squared error (MSE) of the observations from the query set $D^{te}_{\mathcal{T}_j}$.  

We use the same base-learner network as \citet{finn2017model}, i.e., a fully-connected feed-forward network consisting of a single input node followed by two hidden layers with 40 ReLU nodes each and a final single-node output layer.    
We use the MSE loss function to train the model.

In total, we generate $70$K random sine tasks for meta-training, $1$K tasks for meta-validation, and $2$K tasks for meta-testing. 
Every support set $D^{tr}_{\mathcal{T}_j}$ contains $k$ examples, whereas every query set $D^{te}_{\mathcal{T}_j}$ contains $50$ examples to ensure proper evaluation of the inner learning process. 
We perform $30$ runs with different random weight initializations for every experiment and perform meta-validation every $2.5$K tasks.
The best validation performances will be averaged and reported as the validation performance of an algorithm. 
We also include $95$\% confidence intervals. 
The best validation model per run will also be evaluated on the meta-test tasks, for which we adopt a similar evaluation protocol.

\subsubsection{Image classification}

We also use the popular \textit{miniImageNet} benchmark \citep{vinyals2016matching}, with the class splits proposed by \citet{ravi2017optimization}, which were also used by \citet{finn2017model}.
In addition, we also use the CUB benchmark \citep{wah2011caltech}.
Both benchmarks adhere to the $N$-way $k$-shot paradigm described in \autoref{sec:fewshotlearning}.
Following \citet{chen2019closer}, we use $16$ examples per class in every query set.

Moreover, we use the same base-learner network as used by \citet{snell2017prototypical} and \citet{chen2019closer}. 
This network is a stack of four identical convolutional blocks.
Each block consists of $64$ convolutions of size $3 \times 3$, batch normalization, a ReLU nonlinearity, and a 2D max-pooling layer with a kernel size of $2$.
The resulting embeddings of the $84 \times 84 \times 3$ input images are flattened and fed into a dense layer with $N$ nodes (one for every class in a task). 
The base-learner is trained to minimize the cross-entropy loss on the query set, conditioned on the support set.

Following \citet{chen2019closer}, we use $600$ meta-validation tasks, and $600$ meta-test tasks.
Furthermore, when the number of examples per class is $k=1$, we use $60$k meta-training tasks.
For $k=5$, we train on $40$K meta-training tasks. 
We perform $5$ runs with different random weight initializations for every experiment and, as with sine wave regression, we perform meta-validation every $2.5$K tasks.
The same evaluation protocol as used for sine wave regression applies, with the exception that we are now interested in the accuracy instead of the MSE loss.

Importantly, we investigate two main scenarios.
In the first scenario, the \textit{within-distribution} setting, the techniques are evaluated on tasks from the same data set that was used for meta-training (e.g., train on miniImageNet and evaluate on miniImageNet).
In the second scenario, the \textit{out-of-distribution} setting proposed by \citet{chen2019closer}, the techniques are evaluated on tasks from a different data set (CUB) than the one used for meta-training (miniImageNet).

\subsection{Hyperparameter optimization}

\begin{table*}[!h]
\caption{Median meta-test accuracy scores and $95$\% confidence intervals over 5 runs of $5$-way image classification on miniImageNet (left) and CUB (right). The best performance is displayed in bold font. Note that a higher accuracy indicates better performance.}
\label{tab:image}
\vskip 0.15in
\begin{center}
\begin{small}
\begin{sc}
\begin{tabular}{@{\extracolsep{4pt}}lllll}
\toprule
& \multicolumn{2}{c}{miniImageNet} & \multicolumn{2}{c}{CUB} \\
\cline{2-3} \cline{4-5}
{}                          & 1-shot & 5-shot & 1-shot & 5-shot \\ 
\midrule
TrainFromScratch                         & 0.29 $\pm$ 0.00    & 0.40 $\pm$ 0.00    & 0.30 $\pm$ 0.00    & 0.46 $\pm$ 0.00    \\
Finetuning                  & 0.38 $\pm$ 0.00    & 0.56 $\pm$ 0.00    & 0.33 $\pm$ 0.01    & 0.53 $\pm$ 0.01    \\
Baseline++                  & 0.44 $\pm$ 0.00    & 0.58 $\pm$ 0.00    & 0.36 $\pm$ 0.01    & 0.53 $\pm$ 0.01    \\
Meta-learner LSTM                  & 0.45 $\pm$ 0.01    & 0.61 $\pm$ 0.00    & 0.50 $\pm$ 0.00    & 0.65 $\pm$ 0.01    \\
Meta-learner LSTM\tablefootnote{\label{fn:lstm}Our enhanced version of the meta-learner LSTM, which takes raw gradients as inputs, uses second-order gradients, and makes $8$ updates per task.} & \textbf{0.48} $\pm$ 0.01 & 0.63 $\pm$ 0.01 & \textbf{0.53} $\pm$ 0.01 & 0.71 $\pm$ 0.00 \\
MAML       & 0.47 $\pm$ 0.01    & 0.63 $\pm$ 0.00    & 0.52 $\pm$ 0.00    & \textbf{0.73} $\pm$ 0.01    \\
TURTLE & \textbf{0.48} $\pm$ 0.01    & \textbf{0.64} $\pm$ 0.01    & \textbf{0.53} $\pm$ 0.00    & 0.72 $\pm$ 0.01    \\
\bottomrule
\end{tabular}
\end{sc}
\end{small}
\end{center}
\vskip -0.1in
\end{table*}

First, we investigate the effect of the order of information (first- versus second-order), the number of updates $T$ per task, and further increasing the number of layers of the meta-network on the performance of TURTLE on sine wave regression.
The results are displayed in \autoref{fig:layersorder}.
Note that in this experiment, we fixed the learning rate vector $\boldsymbol{\alpha}$ to be a vector of ones, which means that the updates proposed by the meta-network are directly added to the base-learner parameters without any scaling.
Moreover, the only input that the meta-network receives is the gradient of the loss on the support set with respect to a base-learner parameter, and every hidden layer of the meta-network consists of 20 nodes followed by ReLU nonlinearities.

As we can see, the difference between first- and second-order MAML is relatively small, which was also found by \citet{finn2017model}.
In contrast, this is not the case for TURTLE, where the first-order variant fails to achieve a similar performance as second-order TURTLE. 
Furthermore, we see that the stability of TURTLE decreases as $T$ increases. 
Lastly, we find that $5$ or $6$ hidden layers yield the best performance across different values of $T$.
For this reason, all further TURTLE experiments will be conducted with a meta-network of $5$ hidden layers.

In an attempt to stabilize TURTLE and further improve the performance, we experimented with the inclusion of additional input information (\autoref{sec:metanetwork}) and larger meta-batches.
This was done through individual grid searches on top of the meta-network with 5 hidden layers.
The used grids are displayed in \autoref{tab:grids}.  
The best settings were: $(\mathit{no}, \mathit{gradients}, \mathit{0}, \mathit{1} )$, $(\mathit{yes}, \mathit{gradients}, \mathit{0.9}, \mathit{2})$, and $(\mathit{yes}, \mathit{gradients}, \mathit{0.3}, \mathit{4})$ for $T=1$, $T=5$, and $T=10$ updates per task, respectively.

\begin{table}[h]
    \caption{The hyperparameter grids that we used for tuning TURTLE on sine wave regression.}
    \label{tab:grids}
    \vskip 0.1in
    \begin{center}
    \begin{small}
    \begin{sc}
    \begin{tabular}{l|l}
    \toprule 
        Hyperparameter & Grid  \\
        \midrule 
        Time input & yes, no \\
        History & gradients, updates \\
        Decay factor ($\beta$) &  $0, 0.1, 0.2, 0.3, 0.4, 0.6, 0.8, 0.9, 0.95$ \\
        Meta-batch size &  $ 1, 2, 4, 8, 16, 32, 64 $ \\
        \bottomrule
    \end{tabular}
    \end{sc}
    \end{small}
    \end{center}
    \vskip -0.1in
\end{table}

In \autoref{fig:valperf}, we compare the $5$-shot meta-validation performances of these tuned TURTLE models with those of MAML and the meta-learner LSTM (for which we used the hyperparameters reported by the original authors).
As we can see, TURTLE outperforms both MAML and the meta-learner LSTM. 
Lastly, we note that $5$-step TURTLE achieves the best performance. 

\begin{figure}[htb]
\vskip -0.12cm
\begin{center}
\centerline{\includegraphics[width=\columnwidth]{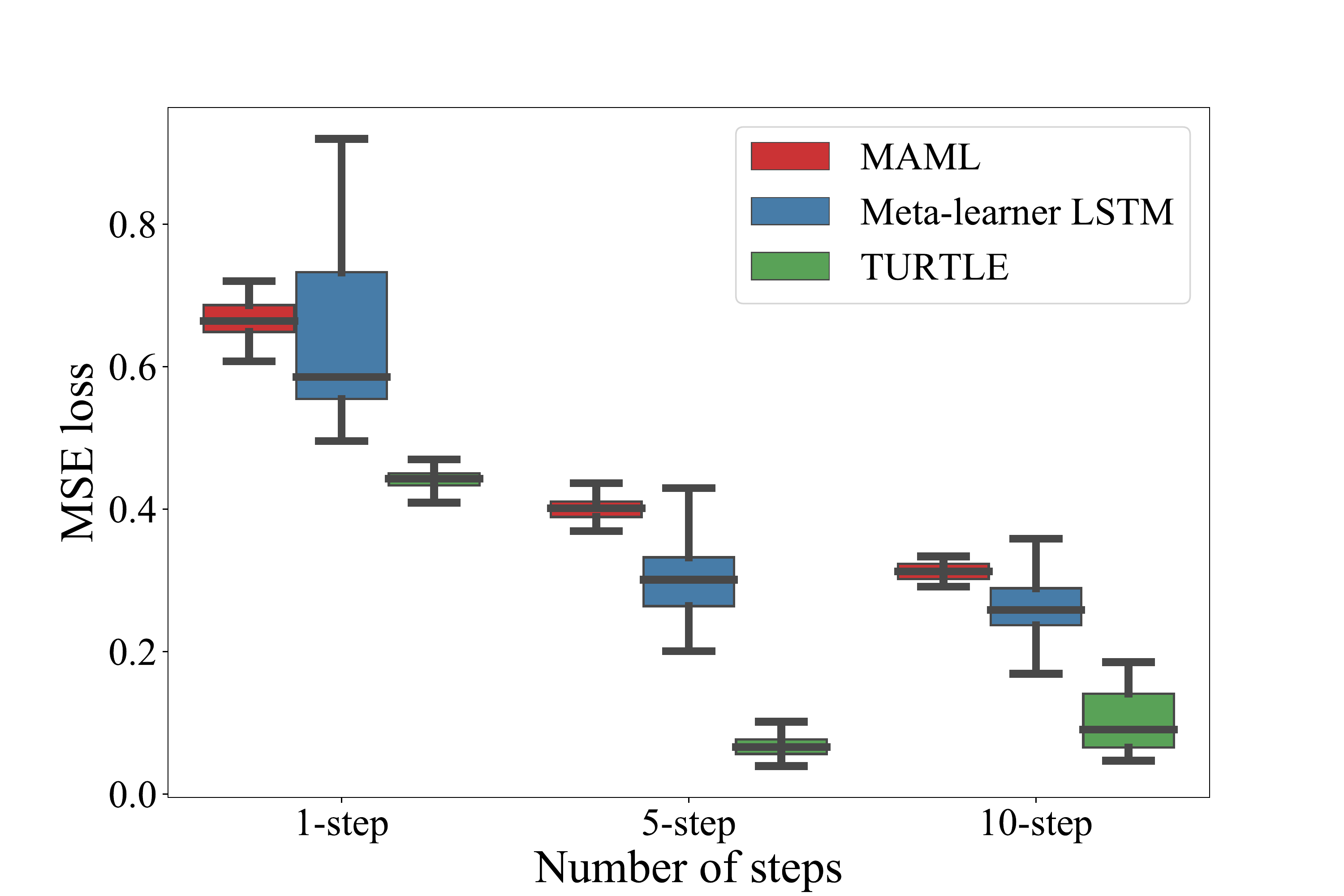}}
\caption{Meta-validation performance of MAML, the meta-learner LSTM, and TURTLE on $5$-shot sine wave regression. Note that a lower error indicates better performance.}
\label{fig:valperf}
\end{center}
\vskip -0.2in
\end{figure}

\subsection{Image classification results}

\textit{Without} additional hyperparameter tuning, we now investigate the performance of 5-step TURTLE on image classification tasks. 
An overview of the  hyperparameters that were used for all techniques can be found in the supplementary material.
We compare the performance against three simple transfer-learning models, following \citet{chen2019closer}: train from scratch, finetuning, and baseline++.
The meta-test accuracy scores on $5$-way miniImageNet and CUB classification are displayed in \autoref{tab:image}.
Note that we use the best-reported hyperparameters for MAML and the meta-learner LSTM on miniImageNet, while we use the best hyperparameters found on sine wave regression for TURTLE.
Based on our hyperparameter experiments for TURTLE, we also investigate an enhanced version of the meta-learner LSTM which uses raw gradients as meta-learner input and second-order information.
On CUB, we use the same hyperparameters as on miniImageNet. 

As we can see, the performances of all models are better on $5$-shot classification compared with $1$-shot classification.
Looking at the results for miniImageNet, we see that TURTLE yields the best performance in three settings. 
On the CUB data set, however, TURTLE is outperformed by MAML in the 5-shot setting, while the reverse holds in the 1-shot setting. 
Moreover, we see that our enhanced version of the meta-learner LSTM achieves better performance than the original meta-learner LSTM, and is on par with TURTLE in two settings.

\subsection{Cross-domain performance and time complexity}

We also investigate the robustness of the meta-learning algorithms when a task distribution shifts occurs.
For this, we train the techniques on miniImageNet and evaluate their performance on CUB tasks, following \citet{chen2019closer}.
The results are shown in \autoref{tab:cross}.
As we can see, TURTLE also achieves the best performance in this challenging scenario.
Furthermore, we see that our enhanced meta-learner LSTM achieves better performance than the original, especially in the 5-shot setting. 

\begin{table}[thb]
\caption{Median meta-test accuracy scores on $5$-way CUB after being trained on miniImageNet. The median accuracy and $95$\% confidence intervals were computed over 5 runs. The meta-learner LSTM$^*$ refers to our enhanced version of the meta-learner LSTM, which takes raw gradients as inputs, uses second-order gradients, and makes $8$ updates per task.}
\label{tab:cross}
\vskip 0.15in
\begin{center}
\begin{small}
\begin{sc}
\begin{tabular}{lcccc}
\toprule
 & 1-shot & 5-shot \\ 
\midrule
TrainFromScratch        & 0.30 $\pm$ 0.00  & 0.46 $\pm$ 0.00 \\
Finetuning & 0.33 $\pm$ 0.00  & 0.52 $\pm$ 0.00  \\
Baseline++ & 0.35 $\pm$ 0.00  & 0.52 $\pm$ 0.00  \\
Meta-learner LSTM & 0.35 $\pm$ 0.00  & 0.51 $\pm$ 0.01  \\
Meta-learner LSTM\footnotemark[2] & 0.36 $\pm$ 0.00  & 0.56 $\pm$ 0.01  \\
MAML       & 0.39 $\pm$ 0.00  & 0.58 $\pm$ 0.00 \\
TURTLE     & \textbf{0.42} $\pm$ 0.01    & \textbf{0.59} $\pm$ 0.01  \\
\bottomrule
\end{tabular}
\end{sc}
\end{small}
\end{center}
\vskip -0.1in
\end{table}

Lastly, we compare the running times of MAML, the meta-learner LSTM, and TURTLE on miniImageNet and CUB. 
A run comprises the time it costs to perform meta-training, meta-validation, and meta-testing on miniImageNet, and evaluation on CUB. 
We measure the average time in full hours across $5$ runs on nodes with a Xeon Gold 6126 2.6GHz 12 core CPU and PNY GeForce RTX 2080TI GPU.
The results are displayed in \autoref{tab:time}.
As we can see, the first-order algorithms (fo-MAML and the meta-learner LSTM) are the fastest, while the second-order algorithms are slower (so-MAML and TURTLE).
Note that the original meta-learner LSTM is slower at 1-shot learning compared with 5-shot learning due to the fact that it makes 12 update steps per task in the former and only 5 in the latter. 
TURTLE is, despite its name, not much slower than the other second-order approach (so-MAML), indicating that the time complexity is dominated by learning the base-learner initialization parameters. 

\begin{table}[!tb]
\caption{Average running time in hours of $5$-step MAML, the meta-learner LSTM, and TURTLE on $5$-way miniImageNet and CUB. The standard deviations across the five runs are given by $\pm x$.}
\label{tab:time}
\vskip 0.15in
\begin{center}
\begin{small}
\begin{sc}
\begin{tabular}{lrr}
\toprule
Algorithm & $1$-shot & $5$-shot \\
\midrule
First-order MAML &  $3 \pm 0$  &  $3.7 \pm 0$  \\
Meta-learner LSTM  &  $4.5 \pm 0.01$  &   $3.90 \pm 0.01$ \\
Meta-learner LSTM\footnotemark[2]   & $8.6 \pm 0.01 $ & $12.3 \pm 0.02$   \\
Second-order MAML &  $5.8 \pm 0.06$  &  $7.9 \pm 0.02$ \\
TURTLE  &  $6 \pm 0.06$  & $8.1 \pm 0.16$ \\
\bottomrule
\end{tabular}
\end{sc}
\end{small}
\end{center}
\vskip -0.2in
\end{table}

\section{Discussion and future work}\label{sec:discuss}

In this work, we have formally shown that the meta-learner LSTM \citep{ravi2017optimization} subsumes MAML \citep{finn2017model}. 
Experiments of \citet{finn2017model} and ourselves, however, show that MAML outperforms the meta-learner LSTM.
We formulated two hypotheses for this surprising finding and, in turn, we formulated a new meta-learning algorithm named TURTLE, which is simpler than the meta-learner LSTM as it is stateless, yet more expressive than MAML because it can learn the weight update rule as it features a separate meta-network.

We empirically demonstrate that TURTLE outperforms both MAML and the meta-learner LSTM on sine wave regression and---without additional hyperparameter tuning---on the frequently used miniImageNet benchmark.
This shows that better update rules exist for fast adaptation than regular gradient descent, which is in line with findings by \citet{andrychowicz2016learning}.

Our hyperparameter analysis on sine wave regression shows that second-order gradients are crucial for achieving good performance with TURTLE. 
In contrast, first-order MAML is a good approximation to second-order MAML as it yields similar performance \citep{finn2017model}.
This finding highlights the distinction between the base- and meta-level goals. 
On the base-level, we wish to find an initialization that is close to the optimal parameters for tasks $\mathcal{T}_j \backsim p(\mathcal{T})$.
Assuming reasonable local smoothness of the base-level landscape, we could ignore our optimization trajectory $\boldsymbol{\theta}_j^{(0)} \rightarrow \boldsymbol{\theta}_j^{(1)} \rightarrow \ldots \rightarrow \boldsymbol{\theta}_j^{(T)}$ as the gradient of the query loss at $\boldsymbol{\theta}_j^{(T)}$ will presumably point towards the direction of the optimal parameters for task $\mathcal{T}_j$, and hence we can move the initialization $\boldsymbol{\theta}$ in that direction. 
On the meta-level, in contrast, we wish to learn an optimization strategy,  which is sequential in nature. 
This implies that an update at time step $t$ influences the gradient inputs that the meta-network will receive at time steps $t < t' < T$.
The consequence of ignoring second-order gradients, as done by the meta-learner LSTM, is that the computation graph becomes disconnected, which makes the meta-network unaware that an update at time step $t$ influences the inputs at future time steps $ t < t' < T$. 

Moreover, we enhanced the meta-learner LSTM by using raw gradients as meta-learner input and second-order information, as they were found to be important for TURTLE. 
Our results indicate that this enhanced version of the meta-learner LSTM systematically outperforms the original technique by $1-6$\% accuracy. 
A promising direction for future work may be to investigate why TURTLE has a slight edge compared with the enhanced meta-learner LSTM.
Another---somewhat related---direction would be to further enhance the meta-learner LSTM, for example, by using meta-batches, as that also increased the training stability of TURTLE.
As a side note, we believe that the performance of TURTLE could also be improved on miniImageNet and CUB by performing hyperparameter tuning on these specific data sets. 

While TURTLE and the enhanced meta-learner LSTM were shown to yield good performance, it has to be noted that this comes at the cost of increased computational expenses compared with first-order algorithms.
That is, these second-order algorithms perform backpropagation through the entire optimization trajectory which requires storing intermediate updates and the computation of second-order gradients. 
While this is also the case for MAML, it has been shown that first-order MAML achieves a similar performance whilst avoiding this expensive backpropagation process.
For TURTLE, however, this is not the case, which means that other approaches should be investigated in order to reduce the computational costs. 
Future research  may draw inspiration from \citet{rajeswaran2019meta} who approximated second-order gradients in order to speed up MAML.

Successfully using meta-learning algorithms in scenarios where task distribution shifts occur remains an important open challenge in the field of meta-learning.
Our cross-domain experiment demonstrates that the learned optimization procedure by TURTLE generalizes to different tasks than the ones seen at training time, which is in line with findings by \citet{andrychowicz2016learning}. 
For this reason, we think that learned optimizers may be an important piece of the puzzle to broaden the applicability of meta-learning techniques to real-world problems. 
Future work can further investigate this hypothesis.

In short, our findings show the benefit of learning an optimizer in addition to the initialization weights and highlight the importance of second-order gradients.

\section*{Acknowledgements}
This work was performed using the compute resources from the Academic Leiden Interdisciplinary Cluster Environment (ALICE) provided by Leiden University.

\bibliography{example_paper}
\bibliographystyle{icml2020}


\end{document}